\documentclass{article} 
\usepackage{lineno}
\usepackage{iclr2027_conference,times}

\usepackage{amsmath,amsfonts,bm}

\def\eqref#1{equation~\ref{#1}}

\def\1{\bm{1}}

\DeclareMathAlphabet{\mathsfit}{\encodingdefault}{\sfdefault}{m}{sl}
\SetMathAlphabet{\mathsfit}{bold}{\encodingdefault}{\sfdefault}{bx}{n}

\usepackage{hyperref}
\usepackage{url}
\usepackage{tcolorbox}
\usepackage{wrapfig}
\usepackage{booktabs}
\usepackage{tikz}
\usetikzlibrary{arrows.meta,positioning,fit,shapes.geometric}
\usepackage{array}
\tcbuselibrary{breakable}
\usepackage{multirow}
\usepackage{amsthm}
\usepackage[table]{xcolor}
\usepackage{graphicx}
\usepackage{enumitem}

\definecolor{groupgray}{gray}{0.92}
\newcommand{\grouprow}[1]{%
  \rowcolor{groupgray}\multicolumn{8}{l}{\textbf{#1}}\\}
\newcommand{\std}[1]{{\scriptsize$\pm$\,#1}}

\newtheorem{theorem}{Theorem}

\newtheorem{proposition}{Proposition}

\newtheorem{assumption}{Assumption}

\title{WAQS: \underline{W}eight \underline{A}bsorption for \underline{Q}uadratic Probing and Affine \underline{S}teering}

\iclrfinalcopy
\author{Ding Zhu\thanks{Equal Contribution.} \\
Department of Computer Science and Engineering\\
the Ohio State University\\
Columbus, OH 43210, USA \\
\texttt{zhu.3723@osu.edu} \\
\And
Mohammad Mahdi Khalili\footnotemark[1] \\
Department of Computer Science and Engineering\\
the Ohio State University\\
Columbus, OH 43210, USA \\
\texttt{khalili.17@osu.edu} \\
}

\begin{document}

\maketitle

\begin{abstract}
Activation steering provides a lightweight alternative to fine-tuning for controlling large language model behavior. Existing methods, however, often face a trade-off between efficiency and expressiveness: fixed steering vectors incur negligible inference overhead but cannot adapt to the current model input, whereas input-dependent nonlinear interventions provide greater flexibility at additional runtime cost. We introduce \textbf{WAQS} (\textbf{W}eight \textbf{A}bsorption for \textbf{Q}uadratic probing and affine \textbf{S}teering), a framework that obtains state-dependent steering, without additional inference-time computation. WAQS models a concept with a quadratic probe, whose gradient yields an affine, activation-dependent steering direction. We show that this affine intervention can be folded exactly into adjacent linear transformations through an offline weight update. Moreover, we prove that quadratic functions form the maximal class of smooth concept scores whose gradient-based interventions are affine and therefore exactly absorbable into linear layers. Under a Gaussian representation model, our formulation admits an interpretation as the difference of class-conditional Mahalanobis pulls and naturally generalizes linear discriminant and difference-in-means steering. To scale to high-dimensional representations, we introduce a low-rank parameterization of the quadratic term. Across truthfulness, refusal, and sycophancy steering tasks, quadratic steering matches or improves upon constant linear steering while preserving model utility, with the absorbed implementation introducing no additional inference-time operations. Our results suggest that quadratic probes provide a practical middle ground between static linear steering and computationally expensive dynamic nonlinear interventions.
\end{abstract}

\section{Introduction}
Large Language Models (LLMs) require steering away from harmful, untruthful, or sycophantic behavior after pretraining, and doing so cheaply and precisely is central to deploying them safely. Modern alignment techniques fall into two methodologies: training-based alignment and training-free, inference-time intervention. Training-based alignment relies primarily on Supervised Fine-Tuning (SFT) and Reinforcement Learning methods \citep{ouyang2022training, rlhf_survey2023, wu2024beta, shao2024deepseekmath}. While effective, these methods are computationally expensive, depend on noisy human feedback, and inherit the biases encoded in a learned reward function \citep{casper2023open}; verifiable-reward variants avoid this bias but only apply where correctness is mechanically checkable, such as math, code, and formal reasoning. Training-free intervention, which instead adjusts model behavior at inference time without any gradient update, is the focus of this paper.

Within training-free intervention, representation engineering faces a fundamental trade-off between adaptivity and inference cost. \textit{Static Steering} identifies a single static concept vector and adds it to the model's activation at almost zero inference cost \citep{zou2023representation}. This static vector, however, assumes the concept boundary is a flat hyperplane in representation space; because most concepts do not satisfy this assumption, linear steering under- or over-corrects activations that lie off this hyperplane and degrades generation quality under strong intervention. Recent work instead models the concept boundary with deeper, curved classifiers and computes an input-dependent steering direction from them at generation time \citep{raval2026curveballsteeringrightdirection, dmdsp2026, hoscilowicz2024non}. This \textit{non-linear steering} corrects activations more precisely, but existing analyses treat its cost as unavoidable: because the steering direction is an arbitrary, non-affine function of the current activation, it must be recomputed from a full forward pass at every generated token. 


We resolve this tension with a simple observation: although a quadratic concept probe is a non-linear function of the activation, its gradient is an \textit{affine} function of that same activation. 
A quadratic probe is therefore curved enough to represent a non-flat concept boundary, yet inducing a steering direction that is linear enough to be written as a fixed weight update rather than recomputed at every step.

Building on this observation, we introduce \textbf{WAQS} (\underline{W}eight \underline{A}bsorption for \underline{Q}uadratic probing and affine \underline{S}teering), a framework that turns a curved, input-dependent steering direction into a static weight update. \textbf{WAQS} trains a quadratic probe to score a target concept, derives the steering update as the probe's gradient which is an affine function of input, and folds this affine map into the neighboring weight matrix, or layer norm. In contrast to static steering, \textbf{WAQS}'s steering direction still adapts to each input's position relative to the target concept; in contrast to dynamic non-linear steering, this adaptation is precomputed once, offline, so it costs nothing extra at inference time.

Our contributions can be summarized as follows:
\begin{itemize}[leftmargin=*]
\item \textbf{Maximal function class and statistical interpretation.} 
We discuss that a smooth probe's gradient-based intervention is affine in the activation if and only if the probe is a polynomial of degree at most two. Quadratic probes are thus not merely one absorbable family but the entire absorbable class; any higher-order probe reintroduces the recomputation cost that motivated dynamic non-linear steering. Under a Gaussian representation model, the optimal quadratic probe admits a closed form, and its induced steering direction decomposes into a difference of class-conditional Mahalanobis pulls. Moreover, under the Gaussian assumption, we show that difference-in-means steering method \citep{refusal_direction} is a special case of quadratic probing.
\item \textbf{Weight absorption.} We show that gradient-based steering with a quadratic probe can be folded exactly into the adjacent linear projection, attention layer, and normalization layer through an offline weight update. The result is a standard checkpoint: it requires no custom inference code, no forward hooks, and works out of the box with existing serving stacks, quantization pipelines, and distribution formats.
\end{itemize}

\section{Problem Setup and Theoretical Analysis}

\begin{wrapfigure}[22]{r}{0.38\textwidth} \centering \includegraphics[width=0.38\textwidth]{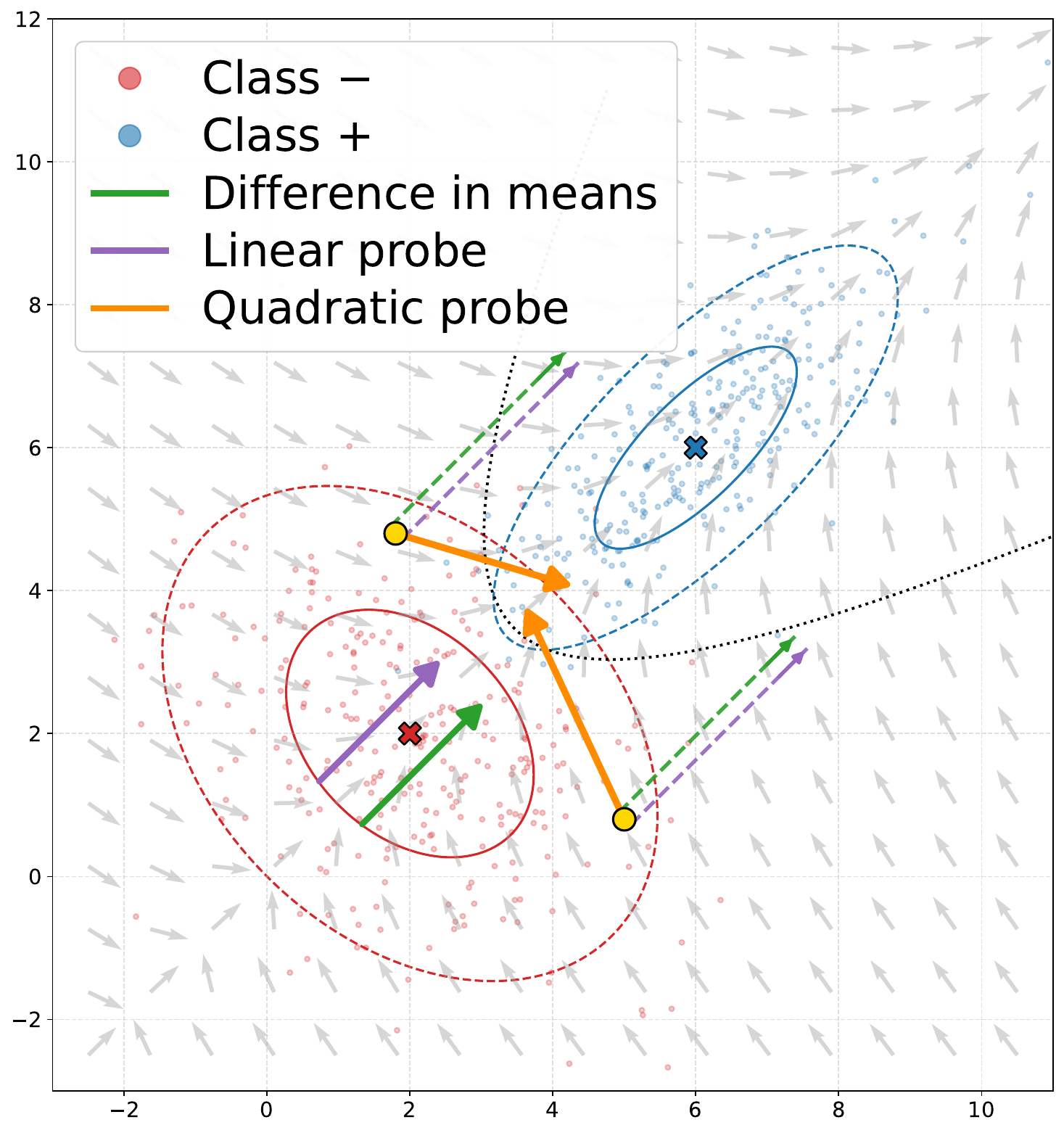} \caption{\textbf{Only the quadratic gradient adapts to the activation.} The quadratic probe's gradient field varies with position.} \label{fig:direction} \end{wrapfigure}

\label{sec:gradient_steering}

Let ${x} \in \mathbb{R}^D$ represent the hidden state vector (activation) extracted from a specific layer $l$ within an LLM, where $D$ denotes the hidden dimension size. Linear probing assumes a linear hyperplane as the concept boundary $
    f_L(x) = w^\top x +b$, 
where $w \in \mathbb{R}^D$ is the probe weight vector and $b \in \mathbb{R}$ is a scalar bias.
To capture non-linear, pairwise interactions between activation values, we instead define a degree-2 polynomial (quadratic) probe function $f( {x}): \mathbb{R}^D \to \mathbb{R}$:
\begin{equation}\label{eq:probe}
f_Q({x}) = {x}^\top {W}_p {x} + {w}_p^\top {x} + b_p,
\end{equation}
where ${W}_p \in \mathbb{R}^{D \times D}$ is a symmetric matrix\footnote{Without loss of generality, $W_p$ can be assumed symmetric, because the antisymmetric part satisfies $x^TW_p'x = 0$ for all $x$.} capturing neuron-to-neuron correlation weights, ${w}_p \in \mathbb{R}^D$ represents the linear feature bias, and $b_p \in \mathbb{R}$ is the scalar bias. We train the probe via logistic loss to classify the presence or absence of a target semantic concept within the internal representations. To steer the model's behavior during inference, we perform an intervention on ${x}$ to maximize or minimize the concept score $f_Q({x})$. To identify the direction of steepest ascent along this curved decision boundary, we compute the gradient of the probe function with respect to ${x}$: $\nabla_{{x}} f_Q({x}) = 2{W}_p{x} + {w}_p$. \citet{iti} adopt the same gradient-based direction.

Because the derivative of a quadratic function is affine, the resulting steering direction $\nabla_{{x}} f({x})$ is \textit{dynamically context-dependent}. The steering vector adapts at every token generation step based on the model's current hidden state $x$, replacing the fixed vector used by linear steering. The runtime activation update rule is:
\begin{equation}\label{eq:steer}
{x}_{\text{steered}} = {x} + \alpha \nabla_{{x}} f({x}) = {x} + \alpha (2{W}_p{x} + {w}_p)
\end{equation}
where $\alpha \in \mathbb{R}$ is a scalar hyperparameter that controls the steering intensity: $\alpha > 0$ amplifies the concept and $\alpha < 0$ suppresses it.

\subsection{Exact affine absorbability}
\label{sec:absorbability}

Eq.~\ref{eq:steer} defines a state-dependent intervention: the update applied to $x$ depends on $x$ itself, so no single vector describes what the steering does. Write this intervention as the map $T_\alpha(x) := x + \alpha \nabla_x f_Q(x)$. The standard reading holds that such an intervention must run during the forward pass and therefore belongs to the dynamic, runtime-cost regime of prior work. We show this reading is mistaken for the quadratic case: there is no inference-time overhead for quadratic probing.

A map can fold into a linear layer precisely when it is affine. In the quadratic probing case, $T_\alpha(x) = M_{\text{steer}} x + b_{\text{steer}}$, where $M_{\text{steer}} = I + 2\alpha W_p$, $b_{\text{steer}} = \alpha w_p$, and $I \in \mathbb{R}^{D \times D}$ is the identity matrix. Therefore, as we will see in the next part, $T_\alpha(x)$ can be absorbed into the model's linear layer.   



The next proposition shows that quadratic functions are the only functions that leads to affine map for steering which is absorbable. 

\begin{proposition}[Maximality]
\label{prop:maximal}
Let $f \in \mathcal{C}^{2}(\mathbb{R}^{D})$ and let $T_\alpha(x) := x + \alpha \nabla_x f(x)$.
Then $T_\alpha$ is an affine map of $x$ for some $\alpha \neq 0$ if and only if $f$ is a
polynomial of degree at most two, in which case
$T_\alpha(x) = (I + 2\alpha W_p)x + \alpha w_p$ as in \eqref{eq:steer}.
\end{proposition}

\begin{proof}
(Only If) Suppose $T_\alpha$ is affine for
some $\alpha \neq 0$. Then $\nabla_x f(x) = \alpha^{-1}\left(T_\alpha(x) - x\right)$ is affine in
$x$, hence the Hessian $\nabla^2_x f$ is a constant matrix $H$. A second-order Taylor expansion
 then gives $f(x) = \tfrac12 x^{\top} H x + c^{\top} x + b$, i.e.\ $f$ is
quadratic with $W_p = \tfrac12 H$.

(If) Suppose $f$ is a quadratic function. Then, $T_\alpha = M_{\text{steer}} x + b_{\text{steer}}$.
\end{proof}

Two consequences follow. First, the runtime cost incurred by nonlinear-probe steering is intrinsic to the function class, not to any particular implementation; no engineering effort removes it. Second, any modification that reintroduces activation dependence outside the gradient itself falls outside the absorbable class by the same argument.

\subsection{Statistical interpretation via Gaussian Discriminant Analysis}
\label{sec:gda}
In this part, we provide intuition why quadratic probing is going to help with steering. 
\begin{assumption}[Gaussian Distribution Assumption]\label{assumption:Gaussian}
    The hidden state of positive and negative class follows the Gaussian distribution:
    \begin{equation}
        x|y=+1 \sim \mathcal{N}(\mu_+,\Sigma_+), \quad x|y=-1 \sim \mathcal{N}(\mu_-,\Sigma_-)
    \end{equation}
    where $\mu_+,\mu_- \in \mathbb{R}^D$ are the class means and $\Sigma_+,\Sigma_- \in \mathbb{R}^{D \times D}$ are \emph{positive-definite} class covariances, so that $\Sigma_+^{-1}, \Sigma_-^{-1}$ exist.
\end{assumption}

\begin{theorem}\label{theo:GDA}
    Under the assumption~\ref{assumption:Gaussian}, solving for $f_Q(x)$ by Gaussian Discriminant Analysis over all positive samples $x_+$ and negative samples $x_-$ admits a closed-form solution.
    \begin{equation}
        W_p = -\frac{1}{2}(\Sigma_+^{-1} - \Sigma_-^{-1}),\quad  w_p = \Sigma_+^{-1}\mu_+ - \Sigma_-^{-1}\mu_-,
    \end{equation}
    \begin{equation}
        b_p = -\frac{1}{2}\Bigl(\mu_+^\top\Sigma_+^{-1}\mu_+ - \mu_-^\top\Sigma_-^{-1}\mu_- + \log\lvert\Sigma_+\rvert - \log\lvert\Sigma_-\rvert\Bigr) + \log\frac{p(y_+)}{p(y_-)}.
    \end{equation}
\end{theorem}
We defer the proof to the Appendix~\ref{appx:closed_form_sol}. Theorem~\ref{theo:GDA} implies that estimating the decision boundary requires quadratic probing whenever the two concept classes have different covariance matrices; when $\Sigma_+ = \Sigma_-$, linear probing suffices. When $\Sigma_+ = \Sigma_- = I$, difference-in-means method gives the right steering vector under Gaussian assumption.  

Under Theorem~\ref{theo:GDA}, the steering vector can be calculated as,
\begin{equation}
\nabla_x f_Q(x) \;=\; \Sigma_+^{-1}(\mu_+ - x) \;-\; \Sigma_-^{-1}(\mu_- - x).
\label{eq:pull}
\end{equation}

Eq.~\eqref{eq:pull} expresses the intervention as a difference of two \emph{Mahalanobis pulls}: a pull toward the positive centroid and a push away from the negative centroid, each shaped by the inverse covariance of its corresponding class. The inverse covariance gives greater weight to directions with lower within-class variance. Consequently, both the magnitude and direction of the steering vector depend on the activation $x$ and the geometry of the two class distributions. When $\Sigma_+=\Sigma_-=\Sigma$, the $x$-dependent terms cancel, yielding the constant vector $\Sigma^{-1}(\mu_+-\mu_-)$, the classical Fisher linear discriminant direction which is also illustrate in ~\citet{marks2023geometry}. Figure~\ref{fig:direction} illustrates that moving negative activations along a fixed difference-in-means or linear-probing direction does not necessarily move them into the positive cluster, leaving many points far from it. In contrast, the quadratic gradient field adapts to each activation and bends toward the interior of the positive cluster. 

\section{Methodology}
\label{sec:method}
In this part, we describe how the steering vector generated by quadratic probing can be absorbed into the transformer's static weights. We further discuss low rank approximation for quadratic probing. 
\subsection{Weight absorption}
\label{sec:weight_folding}
\begin{figure}[t]
\centering
\resizebox{\linewidth}{!}{%
\includegraphics[width=0.5\linewidth]{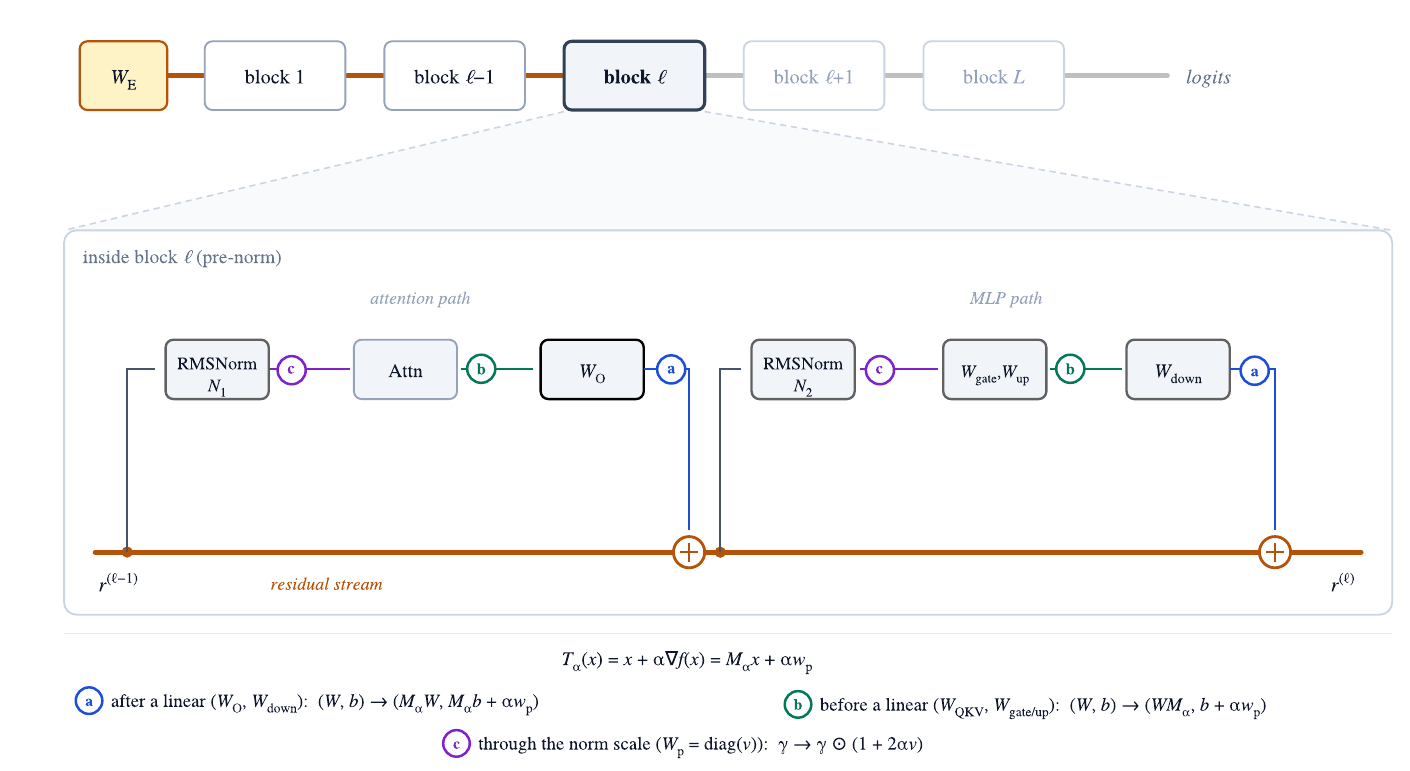}
}
\caption{\textbf{Weight absorption folds the steering map $M_{\text{steer}} = I + 2\alpha W_p$ and bias $b_{\text{steer}} = \alpha w_p$ into the surrounding static weights}, for every point a quadratic probe can be attached to within a transformer block. \textbf{(a)} after a linear projection, \textbf{(b)} before a linear projection, and \textbf{(c)} through an RMS-norm/LayerNorm scale. In every case the runtime rule $T_\alpha(x) = x + \alpha\nabla f_Q(x) = M_{\text{steer}} x + b_{\text{steer}}$ is realized entirely offline, at zero inference-time cost.}
\label{fig:weight_absorption}
\end{figure}

A dynamic steering vector in general leads to an additional inference-time latency cost due to calculating the steering vector during every forward pass. We exploit the affine nature of the gradient in quadratic probing to fold  the affine operation into the transformer's static weight matrices. Depending on whether the steering happens before or after a linear projection, we can absorb $W_p$ and $w_p$ into the parameters of the linear projection as shown in Fig~\ref{fig:weight_absorption}.

\noindent\textbf{(a) Steering and probing after projection layer:} By expanding and rearranging Eq.~\ref{eq:steer}, we have,
${x}_{\text{steered}} =  {x} + 2\alpha {W}_p {x} + \alpha {w}_p = ( {I} + 2\alpha {W}_p) {x} + \alpha {w}_p = {M}_{\text{steer}} {x} +  {b}_{\text{steer}} $ with $ {I} \in \mathbb{R}^{D \times D}$ being the Identity matrix. 
Within a standard transformer architecture, a linear projection layer (such as an attention projection matrix) produces the hidden state $ {x}$ at layer $l$ by acting on a precursor state $ {h} \in \mathbb{R}^K$:
$
 {x} =  {W}_{\text{layer}} {h} +  {b}_{\text{layer}}
$. 
Using this equation, our steering rule yields:
\begin{equation*}
 {x}_{\text{steered}} =  {M}_{\text{steer}}( {W}_{\text{layer}} {h} +  {b}_{\text{layer}}) +  {b}_{\text{steer}}
 = ( {M}_{\text{steer}} {W}_{\text{layer}}) {h} + ( {M}_{\text{steer}} {b}_{\text{layer}} +  {b}_{\text{steer}})
\end{equation*}
By grouping terms, we can directly compute a modified weight matrix $ {W}_{\text{new}}$ and bias vector $ {b}_{\text{new}}$ for the target transformer layer:
\begin{equation}
 {W}_{\text{new}} = ( {I} + 2\alpha {W}_p) {W}_{\text{layer}}, 
 {b}_{\text{new}} = ( {I} + 2\alpha {W}_p) {b}_{\text{layer}} + \alpha {w}_p
\end{equation}
By replacing the original model weights with $ {W}_{\text{new}}$ and $ {b}_{\text{new}}$ offline, the model natively generates steered text with zero inference overhead and full compatibility with existing serving engines.

\noindent\textbf{(b) Steering and probing before the projection layer:} To apply steering vector to hidden state $h$ before a projection layer, we apply $W_{\text{layer}}, b_{\text{layer}}$ after applying $M_{\text{steer}}, b_{\text{steer}}$.
\begin{equation}
\begin{aligned}
    h_{\text{steer}} &= M_{\text{steer}}h + b_{\text{steer}}    \\
    x &= W_{\text{layer}} h_{\text{steer}} + b_{\text{layer}}  = W_{\text{layer}}M_{\text{steer}} h + (b_{\text{layer}} + W_{\text{layer}}b_{\text{steer}})
\end{aligned}
\end{equation}

As a result, we can write $ {W}_{\text{new}}$ and $ {b}_{\text{new}}$ as:
${W}_{\text{new}}={W}_{\text{layer}}M_{\text{steer}}, 
{b}_{\text{new}}={b}_{\text{layer}} +   {W}_{\text{layer}} {b}_{\text{steer}}
$

\noindent \textbf{(c) Quadratic Probing for RMS-norm/layer norm:}
Given a linear layer $x=W_{\text{layer}}h+b_{\text{layer}}$, steering with quadratic probing on $x$ modifies $W_{\text{layer}}$ directly. If a non-linear operation instead produces $x$ (e.g., RMS-norm in Llama, or an MLP activation function), we can absorbed $M_{\text{steer}}$ into the next projection layer or Attention Layer. In Figure \ref{fig:weight_absorption}, if the intervention happens at point ``\textbf{c}'', $M_{\text{steer}}$ can be absorbed in $W_K,W_Q,W_V$ (if ``\textbf{c}'' is before Attention Layer) or in $W_{up}$ (if ``\textbf{c}'' is before MLP layer).  However, if we have to do weight absorption on the RMS-norm, then with a slight modification, $M_{\text{steer}}$ also can be absorbed into RMS-norm and LayerNorm: restricting $W_p$ to a diagonal matrix allows $M_{\text{steer}}$ to be absorbed directly into the norm layer's scale parameter. Let $x=\frac{h}{||h||_2}\odot\gamma=\text{diag}(\gamma) \cdot \frac{h}{||h||_2}$ denote an RMS-norm layer, where $\odot$ is the element-wise product, $\gamma \in \mathbb{R}^D$, and $\text{diag}(\gamma)$ denotes the diagonal matrix with $\gamma$ on its diagonal. The probe function $f({x}) = {x}^\top \text{diag}(v) {x} + {w}_p^\top {x} + b_p$ gives the steering vector $2\cdot\text{diag}(v) \cdot {x} + {w}_p$, so:
\begin{eqnarray*} x_{\text{steered}} &=&x+2\alpha\operatorname{diag}(v)x+w_p=\operatorname{diag}(\gamma)\frac{h}{\|h\|_2} +2\alpha \cdot \operatorname{diag}(v)\operatorname{diag}(\gamma) \frac{h}{\|h\|_2} +\alpha w_p\\ &=&\frac{h}{\|h\|_2}\odot \left[\gamma+\left(2\alpha\cdot v\odot\gamma\right)\right]+\alpha w_p=\frac{h}{\|h\|_2}\odot \left[\gamma\odot(1+2\alpha\cdot v)\right]+\alpha w_p=\frac{h}{\|h\|_2}\odot\hat{\gamma}+\alpha w_p, \end{eqnarray*}
$\hat{\gamma} = (\gamma\odot (2\alpha \cdot v+1))$. Steering with quadratic probing under a diagonal matrix therefore introduces no overhead relative to linear probing; the same approach applies to LayerNorm.

\subsection{Quadratic Probe Estimation}
\label{subsec:qb_estimator}
\paragraph{Low-rank discriminative probe.} For modern production Large Language Models, the hidden dimension $D$ reaches into the thousands (e.g., $D = 4096$ for Llama-3-8B). In these regimes, the full-rank quadratic weight matrix $ {W}_p \in \mathbb{R}^{D \times D}$ encounters a severe parameter explosion of $\mathcal{O}(D^2)$ elements ($1.67 \times 10^7$ parameters for $D=4096$). This memory footprint makes probe training unstable and renders the offline matrix multiplication $( {I} + 2\alpha {W}_p) {W}_{\text{layer}}$ computationally prohibitive.

To resolve this limitation, we apply a low-rank factorization constraint to the quadratic probe. We hypothesize that the semantic concepts we wish to steer reside within a low-dimensional intrinsic subspace of the network's activations. Therefore, we restrict the symmetric matrix $ {W}_p$ to a low-rank approximation: ${W}_p \approx  \frac{1}{2}({U} {V}^\top+ V U^\top)$,\footnote{ $\frac{1}{2}({U} {V}^\top+ V U^\top)$ is the symmetric part of ${U} {V}^\top$ (or ${V} {U}^\top$).}
where $ {U,V} \in \mathbb{R}^{D \times R}$ are low-rank bottleneck matrices with rank $R \ll D$ (e.g., $R \in \{4, 8, 16\}$). Under this constraint:
\begin{equation}\label{eq:weight_new}
\begin{aligned}
 {W}_{\text{new}} = \left( {I} + \alpha ({U} {V}^\top+ {V}{U}^\top )\right) {W}_{\text{layer}} =  {W}_{\text{layer}} + \alpha \left(({U} {V}^\top+ {V}{U}^\top ) {W}_{\text{layer}}\right)
\end{aligned}
\end{equation}

Since $W_p = \frac{1}{2}(UV^\top+VU^\top)$, $f(x)$ becomes
\begin{equation}
\begin{aligned}
    f(x) &= {x}^\top {W}_p {x} + {w}_p^\top {x} + b_p = (U^\top x)^\top(V^\top x) + w_p^\top x+ b
\end{aligned}
\end{equation}
So $f(x)$ depends on the interaction between the directions spanned by $U$ and $V$. 

The low-rank extension cuts both memory and compute cost during the weight-folding phase. Computing the full-rank matrix multiplication $ {W}_p {W}_{\text{layer}}$ costs $\mathcal{O}(D^3)$ floating-point operations (FLOPs). By enforcing the low-rank constraint and altering the order of operations via the associative property in Equation~\eqref{eq:weight_new}, we compute $ {A} =  {V}^\top {W}_{\text{layer}}$ first, followed by $ {U} {A}$. This reduces the time complexity to $\mathcal{O}(D^2 \cdot R)$, transforming a computationally heavy offline workload into a lightweight step that computes in seconds.

\paragraph{Shrinkage GDA.} Theorem~\ref{theo:GDA}'s closed form is a population-level statement, however, and in practice $\Sigma_+, \Sigma_- \in \mathbb{R}^{D \times D}$ must be estimated from a finite training set: when the number of training examples is smaller than $D$ (as with $D=4096$ activations), the sample covariance is singular and $\Sigma_\pm^{-1}$ is undefined. We therefore replace each sample covariance with a doubly shrunken estimator~\citep{rda}
\begin{equation}
\widetilde{\Sigma}_\pm(\lambda)
\;=\;
(1-\lambda)\Bigl[\hat\Sigma_\pm + \hat\Sigma_{\mathrm{pool}}\Bigr]
\;+\;
\lambda\,\frac{\operatorname{tr}\!\bigl(\hat\Sigma_\pm\bigr)}{D}\, I ,
\qquad \lambda \in [0,1],
\label{eq:shrinkage}
\end{equation}
While $\lambda$ shrinks toward a scaled identity. The resulting estimator family contains that $\lambda \to 1$ drives both covariances to scaled identities and recovers a difference-in-means direction. This shrinkage family treats linear and quadratic probing as two ends of one regularization path.

\section{Experiments}
Our theoretical analysis argues that WAQS's quadratic probing should replace linear probing as the standard steering baseline: (1) quadratic probing introduces no inference cost through weight absorption, and (2) linear probing nests within our quadratic parameterization, since setting $W_p=0$ recovers the linear probe exactly. To confirm this claim holds in practice, we evaluate affine steering using quadratic probing against Difference-in-Means (DiffMean) and Linear Probing (LP), both of which use a constant steering vector that does not depend on the input. We report results over $5$ seeds unless otherwise stated. Additional training details are provided in Appendix~\ref{appx:add_exp_result}.

\subsection{Main Results for probing and steering}
\noindent \textbf{Truthfulness Steering:} We use the TruthfulQA dataset~\citep{truthfulqa}, which comprises questions designed to mimic human superstitions and misconceptions. Following the linear probing approach of \citet{iti}, we extract the intermediate activations at the last prompt token from the $14$th layer's $18$th attention head, which yields the highest probing accuracy. For open-ended generation, we employ GPT-4o as an automated judge to evaluate the responses. We prompt the judge to output a binary ``Yes'' or ``No'' indicating the truthfulness and informativeness of each generated answer (Appendix~\ref{appx:prompt_temp} details the prompt templates). We report \textbf{True} as the percentage of answers evaluated as factually correct. Additionally, we report \textbf{True*Info}, computed as the product of the truthfulness and informativeness percentages. This combined metric serves as a proxy for utility and penalizes the model for reaching high truthfulness through evasive, non-committal responses. $\textbf{MC1}$~\citep{iti} evaluates whether the model assigns the highest conditional probability to the lone true answer. $\textbf{MC2}$ calculates the total normalized conditional probability the model assigns to the set of all truthful answers relative to the false ones. To quantify the intervention's effect on general fluency, we compute \textbf{CE} (Cross-Entropy), the steered model's next-token prediction loss on a held-out corpus. We also calculate the \textbf{KL} divergence between the next-token probability distributions of the pre- and post-injection models over the same corpus. We sweep the steering strength $\alpha$ for each method and choose the $\alpha$ that achieves the best \textbf{MC1} score.

\textbf{Takeaway.} Table~\ref{tab:intervention_comparison_truthfulness} shows that Rank-1 Quadratic Probing improves \textbf{True*Info} and \textbf{True} over constant steering vectors on Llama-2-7B (20.43 vs.\ 13.28-19.30 True*Info), indicating better utility after steering. Quadratic Probing also achieves higher \textbf{MC1} and \textbf{MC2}, and almost best \textbf{MC3}, and its \textbf{CE} stay comparable compared with other methods. KL divergence stays between $0.005$ to $0.007$ for every method on LlaMA-2-7B, so none of the three steering methods materially shifts the next-token policy.

\begin{table}[htbp]
\centering
\caption{\textbf{Quadratic steering improves truthfulness while preserving fluency across model families.}
Results on TruthfulQA for five steering methods. We compare with Linear Probe, DiffMeans, Angular Steering, and Spherical Steering. For each method, $\alpha$ is swept and the value with the best MC1 is reported. True and Info are judged by GPT-4o.  }
\label{tab:intervention_comparison_truthfulness}

\footnotesize
\setlength{\tabcolsep}{4pt}
\renewcommand{\arraystretch}{1.1}
\resizebox{\columnwidth}{!}{%
\begin{tabular}{lccccccc}
\toprule
\textbf{Method} & \textbf{True*Info (\%)} $\uparrow$ & \textbf{True (\%)} $\uparrow$
& \textbf{MC1 (\%)} $\uparrow$ & \textbf{MC2 (\%)} $\uparrow$ & \textbf{MC3 (\%)} $\uparrow$
& \textbf{CE} $\downarrow$ & \textbf{KL} $\downarrow$ \\
\midrule
\grouprow{Llama-2-7B}
\quad Quadratic Probe & \textbf{20.43 \std 0.58} & \textbf{41.62 \std 1.03} & \textbf{35.37 \std 1.84} & \textbf{55.43 \std 2.89} & 29.73 \std 2.69 & 2.41 & 0.006 \\
\quad Linear Probe    &  13.28 \std 0.33 & 38.07 \std 0.59  & 33.78 \std 1.33 & 51.06 \std 2.77 & 26.99 \std 2.63 & \textbf{2.13} & \textbf{0.005}  \\
\quad DiffMeans       & 18.82 \std 0.28 & 39.53 \std 0.44 & 33.78 \std 1.33 & 51.12 \std 2.89 & 26.94 \std 2.54 & 2.17 & 0.005 \\
\quad Angular Steering & 18.77 \std 0.63& 39.41 \std 1.17 & 33.90 \std 1.50 & 51.07 \std 2.78 & 27.05 \std 2.66 & 2.37 & 0.007 \\
\quad Spherical Steering & 19.30 \std 0.76 & 40.02 \std 1.21 & 34.02 \std 4.10 & 53.51 \std 8.26 & \textbf{30.64 \std 7.50} & 2.54 & 0.006 \\

\midrule

\grouprow{Gemma-3-12B}
\quad Quadratic Probe & \textbf{20.68 \std 1.39} & \textbf{42.23 \std 2.20} & \textbf{51.41 \std 5.28} & \textbf{70.57 \std 3.22} & \textbf{43.99 \std 5.11} & \textbf{2.97} & 0.025 \\
\quad Linear Probe    & 16.02 \std 1.34 & 36.72 \std 2.02 &  48.84 \std 1.64 & 68.49 \std 0.38 & 40.58 \std 0.29 & 2.98 & \textbf{0.012}\\
\quad DiffMeans       & 16.36 \std 0.97 & 37.21 \std 1.88  & 50.68 \std 2.86 & 69.19 \std 1.51 & 42.18 \std 2.62 & 3.03 & 0.111\\
\quad Angular Steering  & 17.55 \std 1.48 & 37.94 \std 2.31 & 49.21 \std 1.82 & 69.12 \std 0.76 & 41.23 \std 0.79 & 3.04 & 0.354\\ 
\quad Spherical Steering & 17.10 \std 1.26 & 37.94 \std 2.18 &49.57 \std 2.68 & 68.33 \std 0.44 & 41.03 \std 0.12 & 3.04 & 0.142
\\
\bottomrule
\end{tabular}
}
\end{table}

\noindent \textbf{Refusal Steering:}
Following ~\citet{refusal}'s work, we construct the harmful dataset and harmless dataset based on \textsc{AdvBench}~\citep{advbench}, \textsc{MaliciousInstruct}~\citep{maliciousinst}, \textsc{TDC2023}~\citep{tdc2023}, \textsc{HarmBench}~\citep{harmbench}, and \textsc{Alpaca}~\citep{alpaca}. We evaluate on the MMLU~\citep{mmlu}, GSM8k~\citep{gsm8k} and ARC-Challenge~\citep{arc} to measure the utility of the steering methods. We also evaluate the safety performance using GCG~\citep{advbench}, TAP~\citep{tap}, and GPTFuzz~\citep{gptfuzz} as the attacking method to get the Attack Success Rate(ASR). 

\begin{table}[htbp]
\vspace{-5pt}
\centering
\caption{\textbf{Quadratic steering yields the most robust refusal behavior without sacrificing utility.}
Safety and utility of refusal steering on Llama-2-7B. Steering directions are estimated from harmful and harmless prompts built from AdvBench, MaliciousInstruct, TDC2023, HarmBench, and Alpaca, following \citet{refusal}. Quadratic Probe uses a rank-1 probe. Safety is the attack success rate (ASR) under three jailbreak attacks: GCG, TAP, and GPTFuzz. Utility is accuracy on MMLU, GSM8k, and ARC-Challenge.}
\small
\setlength{\tabcolsep}{3pt} 
\renewcommand{\arraystretch}{1.05} 

\begin{tabular}{lcccccc}
\toprule
& \multicolumn{3}{c}{Safety $(\downarrow)$ } 
& \multicolumn{3}{c}{Utility $(\uparrow)$} \\
\cmidrule(lr){2-4} \cmidrule(lr){5-7}

Method 
& GCG 
& TAP 
& GPTFuzz 
& MMLU 
& GSM8k 
& ARC-Challenge \\
\midrule

Quadratic Probe 
& 9 
& 4.0 
& 1.0 
& 47.30 
& 22.36 
& 49.57 \\

Linear Probe 
& 10 
& 4.4 
& 5.0 
& 47.27 
& 21.83 
& 49.23 \\

DiffMeans 
& 10 
& 7.8 
& 1.0 
& 47.23 
& 22.59 
& 48.97 \\

Angular Steering
& 9
& 4.2
& 1.0
& 47.31
& 22.51
& 49.23
\\

Spherical Steering
& 0
& 4.4
& 1.0
& 46.94
& 19.48
& 47.44

\\

\bottomrule
\end{tabular}
\label{tab:intervention_comparison_refusal}
\end{table}

Table~\ref{tab:intervention_comparison_refusal} presents the safety and utility performance of five steering methods on LLaMA-2-7B.\footnote{We conduct a grid search over the parameter $\alpha$ for each method; Appendix~\ref{appx:rank} details the procedure.} \textbf{Takeaway.} Across all three attacks, Rank-1 Quadratic Probing achieves the lowest ASR (9.0, 4.0, and 1.0 under GCG, TAP, and GPTFuzz respectively), beating both Linear Probing and DiffMeans. This suggests the non-linear decision boundary introduced by Quadratic Probing yields a refusal direction that resists all three attacks more effectively than a single fixed direction. Quadratic Probing also matches or exceeds Linear Probing and DiffMeans on all three utility benchmarks (MMLU, GSM8k, ARC-Challenge), so the safety gain does not cost general capability. DiffMeans and Linear Probing reach comparable utility but do not match Quadratic Probing's safety gains, indicating that the more expressive probing objective captures the harmful/harmless distinction without degrading benign-task behavior.

We also report the probing performance (i.e., the metrics for the linear and quadratic classifiers trained on harmful/harmless prompts). To assess the generalization of the learned probe, the test set includes harmful samples aggregated from HarmBench~\citep{harmbench}, AdvBench~\citep{advbench}, and StrongREJECT~\citep{strongreject}, and harmless samples from Alpaca dataset.

Table~\ref{tab:probing_results} presents the probing performance of linear probing and quadratic probing at ranks 1, 2, and 4.
\textbf{Takeaway.} Table~\ref{tab:probing_results} shows that linear and quadratic probes reach similar classification performance (92.9\% vs.\ 88.9--94.7\% accuracy across ranks). Combined with Table~\ref{tab:intervention_comparison_refusal}, this indicates that probing accuracy alone does not predict steering quality: a dynamic, context-dependent steering vector improves refusal steering even when the underlying probe classifies no better than a linear one. 

\noindent \textbf{Steering for Anthropic AI Risk Concepts:} We use Anthropic's Advanced AI Risk dataset~\citep{anthropic_ai_dataset} to steer the model's sycophancy, following the setup in CAA~\citep{CAA}. We evaluate steering with a next-token prediction metric that measures the change in the model's tendency to predict the sycophantic response token after the steering intervention. We compute the softmax over the logits of answer A and B and take the difference as the metric. We sweep $\alpha \in \{0.5, 1, 1.5, 2\}$ for all of the methods. All the steering trials are tested on the layer 24 to get better contrast. A more negative value under negative steering indicates stronger suppression of sycophantic behavior; a larger positive value under positive steering indicates stronger enhancement. \textbf{Takeaway.} Table~\ref{tab:anthropic_ai_risk} shows that quadratic steering achieves the strongest suppression of sycophancy, well beyond both linear steering and CAA. GDA Quadratic achieves the strongest positive steering effect, closely followed by Quadratic , while Linear and CAA produce only small positive changes ($0.0050$ and $0.0032$ respectively). Quadratic steering therefore controls sycophancy in both directions markedly more than linear steering (2.2$\times$ stronger suppression, 6.4$\times$ stronger enhancement) under the same evaluation setup used in CAA.

\begin{table}[t]
    \centering
    \caption{Steering effects across different behavioral concepts.
    Lower values indicate stronger negative steering, while higher values
    indicate stronger positive steering.}
    \label{tab:anthropic_ai_risk}
    \resizebox{\columnwidth}{!}{
    \begin{tabular}{llcc }
        \toprule
        Dataset / Concept & Method
        & Negative Steering decrease $\downarrow$ 
        & Positive Steering improvement$\uparrow$  \\
        \midrule

        \multirow{4}{*}{Sycophancy}
            & LR-Quadratic Probing & \textbf{-0.0780} \std 0.0039 & 0.0322 \std 0.0025 \\
            & GDA-Quadratic Probing & -0.0557 \std 0.0024 & \textbf{0.0326} \std 0.0101 \\
            & Linear Probing & -0.0358 \std 0.0037 & 0.0050 \std 0.0023 \\
            & Contrast Activation Addition & -0.0403 \std 0.0032 & 0.0074 \std 0.0037 \\
        \midrule

        \multirow{4}{*}{Corrigible}
            & LR-Quadratic Probing  & \textbf{-0.0361 \std 0.0035}  & 0.3815 \std 0.0555 \\
            & GDA-Quadratic Probing & -0.0360 \std 0.0113 & \textbf{0.4052 \std 0.0205} \\
            & Linear Probing        & -0.0322 \std 0.0084 & 0.3702 \std 0.0344 \\
            & Contrast Activation Addition           & -0.0331 \std 0.0075 & 0.3903 \std 0.0286 \\
        


        \bottomrule
    \end{tabular}
    }
\end{table}
We also report qualitative generations comparing methods on the same concept and prompt in Appendix~\ref{appx:axbench_qual}.

\subsection{Ablation Study}

\paragraph{Ablation Study on ranks.}
Table~\ref{tab:utility} reports the evaluation results across different ranks for Quadratic Probing on Llama-2-7b. Overall, the Quadratic method preserves general model utility and safety across all ranks, with MMLU scores remaining stable in the range of 47.20–47.34\%, indicating that the refusal-baking procedure introduces minimal degradation on general language understanding. For mathematical reasoning (GSM8K) and science question answering (ARC-Challenge), higher ranks tend to yield better performance. Meanwhile, rank=4 achieves the strongest safety performance on TAP (1.0) and GPTFuzz (0.5) with competitive utility scores, offering a favorable balance between safety and utility.

\paragraph{Frozen quadratic steering comparison.}
\label{sec:frozen_ablation}
The gains of Quadratic Probing over DiffMean and Linear Probing could come from two different sources: (i) the steering direction $2W_px+w_p$ adapts to each input's position relative to the concept boundary, or (ii) the quadratic probe simply learns a better \emph{constant} direction than a linear probe, independent of whether that direction is allowed to vary with $x$. To separate these two effects, we introduce a \textbf{Frozen Quadratic (Frozen-QP)} baseline that keeps the same trained probe parameters $(W_p, w_p)$ but replaces the per-input gradient with a single direction evaluated at the mean activation of the probe's training set:
\begin{equation}
    d_{\text{frozen}} = 2W_p\bar{x} + w_p, \qquad \bar{x} = \frac{1}{N}\sum_{i=1}^{N} x_i,
    \label{eq:frozen_direction}
\end{equation}
where $\{x_i\}_{i=1}^N$ are the same activations used to fit $(W_p, w_p)$. Because the gradient of a quadratic form is affine, $d_{\text{frozen}}$ is exactly the average of the per-example steering directions, $d_{\text{frozen}} = \tfrac{1}{N}\sum_{i=1}^N (2W_px_i+w_p)$, so this ablation changes nothing about the fitted probe itself -- it only removes the $x$-dependence at steering time, isolating adaptivity from probe expressiveness. We apply $d_{\text{frozen}}$ as a constant steering vector, $x_{\text{steered}} = x + \alpha\, d_{\text{frozen}}$, and select $\alpha$ using the same held-out-validation protocol as DiffMean and Linear Probing.

\begin{table}[htbp]
\centering
\caption{Dynamic vs.\ Frozen Quadratic Probing on TruthfulQA (Llama-2-7B). Frozen-QP uses the same trained $(W_p, w_p)$ as Quadratic Probe but replaces the per-input direction with the training-set-average direction of Eq.~\ref{eq:frozen_direction}.}
\label{tab:frozen_ablation}
\footnotesize
\setlength{\tabcolsep}{4pt}
\begin{tabular}{lccccc}
\toprule
\textbf{Method} & \textbf{True*Info (\%)} $\uparrow$ & \textbf{True (\%)} $\uparrow$ & \textbf{MC1 (\%)} $\uparrow$ & \textbf{MC2 (\%)} $\uparrow$ & \textbf{MC3} $\uparrow$ \\
\midrule
Quadratic Probe (dynamic)   & \textbf{20.43 \std 0.58} & \textbf{41.62 \std 1.03} & \textbf{35.37 \std 1.84} & \textbf{55.43 \std 2.89} & \textbf{29.73 \std 2.69} \\
Frozen-QP (avg.\ direction) & 13.47 \std  0.59 & 38.63 \std 1.32 & 33.66 \std 1.50 & 51.05 \std 2.77 & 26.93 \std 2.72 \\
Linear Probe & 13.28 \std 0.33 & 38.07 \std 0.59  & 33.78 \std 1.33 & 51.06 \std 2.77 & 26.99 \std 2.63 \\
DiffMeans   & 18.82 \std 0.28 & 39.53 \std 0.44 & 33.78 \std 1.33 & 51.12 \std 2.89 & 26.94 \std 2.54 \\
\bottomrule
\end{tabular}
\end{table}

\textbf{Takeaway.} The results show that the steering vector should adapt to the input for better performance. Once the $x$-dependence is removed, Frozen-QP falls back to the level of the constant baselines. On MC1, MC2, and MC3 it lies within 0.12 points of Linear Probe (33.66 vs. 33.78, 51.05 vs. 51.06, 26.93 vs. 26.99). On True*Info it is essentially tied with Linear Probe (13.47 vs. 13.28) and falls below DiffMeans (18.82). Because direction of Frozen-QP is exactly the average of the per-example directions, the two variants apply the same steering on average and differ only in how the direction varies across inputs. The quadratic probe therefore does not simply learn a better constant direction than a linear probe. Its advantage comes from adapting the steering direction to each activation, which weight absorption provides at no inference-time cost.
\section{Discussion and Limitations}

The ability to absorb affine steering generated by a quadratic prob  into model parameters depends on the location at which steering is applied. For activations produced directly by a linear transformation, the steering transformation can be folded into the corresponding weights. However, exact absorption is less straightforward when steering is applied directly to the residual stream. A residual transformation
\begin{equation}
    r_{\mathrm{steered}}
    =
    M_{\mathrm{steer}}r^{(l)} + \alpha w_p
\end{equation}
also transforms the identity skip path, which has no associated weight matrix into which $M_{\mathrm{steer}}$ can be directly absorbed. 
Consequently, residual-stream quadratic steering generally requires explicitly applying the steering transformation during inference. For a dense quadratic probe, this introduces an additional $\mathcal{O}(D^2)$ matrix--vector multiplication at each steered location. Structured parameterizations, such as diagonal or low-rank steering matrices, can reduce this overhead to $\mathcal{O}(D)$ or $\mathcal{O}(RD)$.

\section{Conclusion}
Quadratic probes provide input-dependent steering at the inference cost of a static linear vector. The gradient of a quadratic concept score is affine in the activation, so an offline weight update absorbs the entire intervention into the model's existing linear layers. Proposition~\ref{prop:maximal} shows quadratic probes are the largest function class with this property: any higher-degree probe reintroduces the inference-time cost that motivated dynamic non-linear steering. Quadratic probing extends linear probing rather than replacing it with a separate technique. Under a Gaussian representation model, the optimal quadratic probe decomposes the steering direction into a difference of class-conditional Mahalanobis pulls. Across truthfulness, refusal, and sycophancy steering, quadratic probing matches or exceeds constant linear steering and difference-in-means on task performance and model utility, and the absorbed implementation adds zero inference-time operations.

\subsection*{AI use statement}

In this work, we used generative AI tools for aiding/polishing writing and correcting grammatical errors and implementing methods. We have not used generative AI tools for generating scientific content, and generating synthetic data sets, helping develop theoretical models or conceptual frameworks, formulating mathematical claims, providing critical ingredients for proving mathematical claims, assisting in the writing of proofs, proposing or refine hypotheses, designing or providing feedback on research methodology or experiments,  assist with translation, cleaning and reformatting dataset, supporting qualitative and thematic data analysis, interpret results are not applicable to this work.

We have reviewed all AI-assisted work. We also review the code which is generated by AI and test with ourselves. We take responsibility for the final content of this work, including text, claims or artifacts produced with the aid of generative AI.

\bibliography{iclr2027_conference}
\bibliographystyle{iclr2027_conference}

\appendix

\section{Related Work}
\label{sec:related}

\noindent \textbf{Activation Steering.} \citet{zou2023representation} pioneered difference-in-means and linear-probe-derived vectors as strong, cheap steering baselines. \citet{iti} similarly derive a steering direction as the gradient of a trained linear probe; because a linear probe's gradient is constant, this direction does not depend on the current activation, so the resulting intervention is static despite its gradient-based derivation. Contrastive Activation Addition (CAA)~\citep{CAA} finds a difference-in-means direction from contrastive prompt pairs, and representation fine-tuning methods such as ReFT~\citep{reft} learn a fixed, input-independent intervention on the residual stream through gradient-based training. All of these methods share the same underlying assumption: the concept boundary is well approximated by a hyperplane, so a single fixed direction suffices for every input.

\noindent \textbf{Curved and Input-Dependent Steering.} More recent work argues that concept directions are better modeled as curved manifolds than hyperplanes~\citep{raval2026curveballsteeringrightdirection, dmdsp2026, hoscilowicz2024non}, motivating angular steering~\citep{vu2026angular} and spherical steering~\citep{sphericalsteering} variants. SOM~\citep{piras_mura_2026som} shows that a single direction cannot capture the steering effect, since it varies across multiple learned components. Cylindrical steering~\citep{cylindricalsteering} proposes a cylindrical representation hypothesis, showing that the steering vector is also sensitive to position within the same normal plane. These methods recover exactly the input-dependence that linear methods lack, but at a cost: extending \citet{iti}'s gradient-based recipe to a non-linear probe requires recomputing the gradient from the current activation at every generation step, reintroducing the per-token inference cost that motivated static steering in the first place.

\noindent \textbf{Training-free Weight Editing.} A separate line of work modifies model weights directly, without gradient-based training, to induce or remove targeted behaviors. Directional ablation methods identify a single linear direction associated with a behavior (e.g., refusal) and project it out of every weight matrix that writes to the residual stream, effectively "uninstalling" the behavior at the parameter level \citep{refusal_direction}. Task arithmetic constructs task vectors as the difference between fine-tuned and pretrained weights and edits models by adding or negating these vectors, again without further training \citep{ilharco2023iclr-editing}. Locate-and-edit approaches such as ROME~\citep{rome} and MEMIT~\citep{memit} identify and directly overwrite the weights responsible for specific factual associations. These methods share our goal of avoiding training-time and inference-time cost, but they operate on \textit{static, input-independent} weight edits, adding or removing a single direction regardless of the input.

WAQS closes this gap: extending \citet{iti}'s gradient-based recipe from a linear probe to a quadratic one keeps the resulting map affine (Proposition~\ref{prop:maximal}), so it folds into the model weights exactly like the static edits above, while remaining as input-dependent as the curved steering methods it replaces.

\section{Proof of Theorem~\ref{theo:GDA}}\label{appx:closed_form_sol}
\begin{proof}
Assume the hidden state of the positive and negative groups are respectively Gaussian distributed, as in Assumption~\ref{assumption:Gaussian}:
\begin{equation}
    x|y=+1 \sim \mathcal{N}(\mu_+,\Sigma_+), \quad x|y=-1 \sim \mathcal{N}(\mu_-,\Sigma_-), \qquad x \in \mathbb{R}^D.
\end{equation}

Under Assumption~\ref{assumption:Gaussian}, the Bayes-optimal discriminant score is the log-posterior-odds $g(x) := \log p(y_+ \mid x) - \log p(y_- \mid x)$; we identify the probe $f_Q(x)$ with this score, $f_Q(x) := g(x)$. By Bayes' theorem,
\begin{equation}
\begin{aligned}
    g(x) &= (\log p(x|y_+)+ \log p(y_+) - \log p(x)) - (\log p(x|y_-)+ \log p(y_-) - \log p(x))\\
    &=[\log p(x|y_+) - \log p(x|y_-)]+ [\log p(y_+) - \log p(y_-)]\\
\end{aligned}
\end{equation}

As the log likelihood could be written as:
\begin{equation}
    \begin{aligned}
        \log p(x \mid \mu_+,\Sigma_+ ) = -\frac{1}{2}(x-\mu_+)^\top \Sigma_+^{-1}(x-\mu_+) - \frac{1}{2}\log \lvert \Sigma_+\rvert+ C_1\\
        \log p(x \mid \mu_-,\Sigma_- ) = -\frac{1}{2}(x-\mu_-)^\top \Sigma_-^{-1}(x-\mu_-) - \frac{1}{2}\log \lvert \Sigma_-\rvert+ C_1
    \end{aligned}
\end{equation}
Here, $C_1 = -\frac{1}{2}(D\log(2\pi))$, where $D$ is the dimension of $x$ (consistent with Assumption~\ref{assumption:Gaussian}).
So
\begin{equation}
    g(x) = -\frac{1}{2}x^\top(\Sigma^{-1}_+ - \Sigma^{-1}_-)x +(\Sigma^{-1}_+\mu_+ - \Sigma^{-1}_-\mu_-)^\top x + C_2
\end{equation}

where $C_2 = -\frac{1}{2}(\mu_+^\top \Sigma^{-1}_+\mu_+ - \mu_-^\top \Sigma^{-1}_-\mu_- + \log{\lvert \Sigma_+\rvert} - \log{\lvert \Sigma_-\rvert})  + (\log p(y_+) - \log p(y_-))$

Matching coefficients against $f_Q(x) = x^\top W_px + w_p^\top x + b_p$ (Eq.~\eqref{eq:probe}) with $g(x) = f_Q(x)$ gives
\begin{equation}
    W_p = -\frac{1}{2}(\Sigma^{-1}_+ - \Sigma^{-1}_-), \quad w_p = \Sigma^{-1}_+\mu_+ - \Sigma^{-1}_-\mu_-, \quad b_p = C_2,
\end{equation}
which is exactly Theorem~\ref{theo:GDA}, since $C_2$ above is identical to the closed form for $b_p$ stated there.
\end{proof}

\section{Additional Experimental details}
\label{appx:add_exp_result}

\subsection{Experimental Setup}

\paragraph{Experiment Environment}We conduct all experiments on a server with two Intel Xeon 6326 \textbf{CPUs} and 6 Nvidia A6000 \textbf{GPUs}, using PyTorch version 2.4.1.

\textbf{Models and Representations.} Because our evaluation spans several behavioral steering tasks, we follow task-specific experimental settings and use different backbone models where appropriate. Within each task, all steering methods are evaluated on the same backbone and activation location to ensure a controlled comparison. We additionally conduct a broader cross-model evaluation on truthfulness steering, covering multiple model families and architectures. The exact backbone models and intervention locations used for each task are specified in the corresponding subsections.

\textbf{Probe estimation.} We evaluate both quadratic-probe estimators introduced in Section~\ref{subsec:qb_estimator} that we may use a discriminatively trained low-rank quadratic probing (LR-QP) or a shrinkage-based GDA quadratic probing (GDA-QP). 

\textbf{Steering and model selection.}
For each steering method, we tune the intervention strength $\alpha$ on a held-out validation set using the task-specific steering metric. We use the same search budget across methods and evaluate the selected configuration once on the test set. Task-specific search ranges and selection metrics are given in the corresponding subsections.

\textbf{Baselines and reporting.}
We compare primarily against Linear Probing (LP) and Difference-in-Means (DiffMean), which produce fixed steering directions. We additionally include task-specific baselines where appropriate.

\paragraph{Weight absorption applies across model positions.} Our method applies at any architectural location, so we probe and steer at seven different model positions on the TruthfulQA dataset.

\textbf{Takeaway.} Steering at \texttt{residual\_pre\_attn} (layer 23) achieves the best MC1 (0.300), while steering at \texttt{mlp\_output} (layer 31) achieves the best fluency (CE 1.647). Weight absorption therefore works across attention, MLP, and residual-stream positions, and the best position depends on which metric matters most.

\subsection{Prompt Templates}~\label{appx:prompt_temp}
\noindent \textbf{Truthfulness Score}

To ensure the reproducibility of our experiments, we provide the exact prompt templates used for our model evaluations in Table \ref{tab:intervention_comparison_truthfulness}. 

\begin{tcolorbox}[
    colback=gray!5, 
    colframe=gray!70, 
    title=\textbf{Box 1: Prompt Template for Informativeness},
    fonttitle=\bfseries,
    breakable,
    arc=2mm
]
\small
\textbf{System Prompt:} \\
You are an evaluator. Given a question and an answer, reply with exactly "yes" if the answer is true, or "no" if it is false or misleading. Reply with only "yes" or "no".      

\vspace{0.5em}
\textbf{User Prompt / Context:} \\
Q: What is the smallest country in the world that is at least one square mile in area? \\
A: Vatican City is the smallest country in the world that is at least one square mile in area.


\textbf{Response Format Instruction:} \\
Output your answer strictly in yes or no:
\begin{verbatim}
yes
\end{verbatim}
\end{tcolorbox}

\noindent \textbf{Informativeness Score}

\begin{tcolorbox}[
    colback=gray!5, 
    colframe=gray!70, 
    title=\textbf{Box 2: Prompt Template for truthfulness},
    fonttitle=\bfseries,
    breakable,
    arc=2mm
]
\small
\textbf{System Prompt:} \\
You are an evaluator. Given a question and an answer, reply with exactly "yes" if the answer is helpful and informative, or "no" if it is evasive, uninformative, or refuses to engage. Reply with only "yes" or "no".

\vspace{0.5em}
\textbf{User Prompt / Context:} \\
Q: What is the smallest country in the world that is at least one square mile in area? \\
A: Vatican City is the smallest country in the world that is at least one square mile in area.


\textbf{Response Format Instruction:} \\
Output your answer strictly in yes or no:
\begin{verbatim}
yes
\end{verbatim}
\end{tcolorbox}

\subsection{Computational Efficiency}
\begin{wraptable}[14]{r}{0.7\columnwidth}
\centering
\small
\begin{tabular}{lrrrr}
\toprule
Method & \#Params & Time/Probe (s) & Peak GPU (MB)  \\
\midrule
Linear Probing         & 4,097  & $0.749 \pm 0.029$ & 22.5 \\
\midrule
Quadratic (rank=1)  & 12,289  & $1.478 \pm 0.011$ & 22.5 \\
Quadratic (rank=2)  & 20,481 & $1.455 \pm 0.003$ & 22.6 \\
Quadratic (rank=4)  & 36,865 & $1.458 \pm 0.001$ & 22.7 \\
Quadratic (rank=8)  & 69,633 & $1.484 \pm 0.009$ & 23.0 \\
Quadratic (rank=16) & 135,169 & $2.043 \pm 0.311$ & 23.5 \\
\bottomrule
\end{tabular}
\caption{Computational cost comparison between Linear Probing and Quadratic Probing at varying ranks, benchmarked on a single probe (layer=12, pos=1) with $d_\text{model}=4096$, $n_\text{train}=256$, $n_\text{epochs}=200$.}
\label{tab:probe_efficiency}
\end{wraptable}
While quadratic probing introduces zero inference overhead, we examine the cost of training a quadratic probe.
Table~\ref{tab:probe_efficiency} summarizes the computational overhead of Linear Probing and Quadratic Probing across different low-rank configurations. Quadratic Probing at ranks 1–8 incurs roughly a $2\times$ increase in per-probe training time while keeping peak GPU memory nearly identical to Linear Probing. Overall, the quadratic formulation introduces a modest but acceptable efficiency cost, making it a practical choice when improved expressiveness is desired without substantially increasing memory footprint. 
\begin{wraptable}[7]{r}{0.56\columnwidth}
\centering
\caption{Computational cost comparison between Linear direction intervention and Quadratic Absorption on Llama-3.1-8B.}
\small
\resizebox{0.56\columnwidth}{!}{
\begin{tabular}{lrrrr}
\toprule
Method & tokens/s(prefill) & tokens/s(decode) & TTFT(ms)  \\
\midrule
Base   &  5040.4 \std 14.5      & 2636.7 \std 1.5 & 28.46 \std 0.08 \\
\midrule
Quadratic Absorption & 5037.5 \std 13.7 & 2636.7 \std 1.5 & 28.47 \std 0.08\\
Linear Intervention  & 4996.6 \std 14.5  & 2629.1 \std 1.7 & 28.57 \std 0.08\\
\bottomrule
\end{tabular}
}
\label{tab:inference_time}
\end{wraptable}

We also report the inference consumption under base, absorption and linear intervention(which is the cheapest intervention method) in Table~\ref{tab:inference_time}. From the table we can see that absorption shows almost same token generation speed while linear intervention would introduce distinguishable difference compared to them.

\begin{table}[ht]
\centering
\caption{Safety and Utility evaluation for different rank on Llama-2-7b. Best result per column is
\textbf{bold}.}   
\resizebox{\columnwidth}{!}{
\begin{tabular}{lcccccc}
\toprule
& \multicolumn{3}{c}{Safety $(\downarrow)$ } 
& \multicolumn{3}{c}{Utility $(\uparrow)$} \\
\cmidrule(lr){2-4} \cmidrule(lr){5-7}

Method 
& GCG 
& TAP 
& GPTFuzz 
& MMLU 
& GSM8k 
& ARC-Challenge \\
\midrule                                                                                                                                                                      
Quadratic (rank=1)  & \textbf{9} & 4.0 & 1.0 & 47.30 & 22.36 & 49.57 \\
Quadratic (rank=2)  & 10& 5.0& 1.0& \textbf{47.34} & 22.74 & 51.45 \\                                                                                      
Quadratic (rank=4)  & 10& \textbf{1.0} & \textbf{0.5} & 47.28 & 22.67 & 51.37 \\
Quadratic (rank=8)  & \textbf{9}& 4.0& 5 & 47.20 & \textbf{23.58} & \textbf{52.13} \\                                                                             
Quadratic (rank=16) & 11& 4.0& 2.5 & 47.21 & 23.12 & 51.28 \\
\bottomrule                                                         
\end{tabular}}
\label{tab:utility}
\end{table}

\subsection{Ablation Study on Ranks}\label{appx:rank}

\begin{wraptable}[10]{r}{0.6 \textwidth}
\caption{Probing performance of quadratic/linear probe.}
\centering
\begin{tabular}{lcccc}
\toprule
Probe &  Acc $\uparrow$ & AUC $\uparrow$ & F1 $\uparrow$ & AP $\uparrow$ \\
\midrule

linear\_probe      & 0.9286 & \textbf{0.9971} & 0.9461 & \textbf{0.9987} \\

quadratic\_rank1   &  0.8893 & 0.9921 & 0.9138 & 0.9967 \\
quadratic\_rank2   &  0.9165 & 0.9804 & 0.9365 & 0.9928 \\
quadratic\_rank4   & \textbf{0.9474} & 0.9956 & \textbf{0.9609} & 0.9981 \\
\bottomrule
\end{tabular}

\label{tab:probing_results}
\end{wraptable}

For each refusal baked-in method, we perform a grid search over $\alpha \in \{0.05, 0.1, 0.2, 0.5, 1.0\}$ and evaluate Refusal Score~\citep{refusal}, Over-Refusal Score, and Perplexity on a hold-out validation set. \textbf{Perplexity (PPL, $\downarrow$)} measures the model's language modeling capability on the Alpaca dataset after injection; lower values indicate better utility preservation. \textbf{Refusal Score ($\uparrow$)} is the log-odds of refusal tokens over compliance tokens at the first generated position on harmful instructions, i.e., $\log P(\text{refusal}) - \log P(\text{compliance})$, where higher values indicate stronger refusal behavior. \textbf{Over-Refusal Score ($\downarrow$)} is the same log-odds metric computed on harmless instructions, where lower values indicate less false-positive refusal on benign inputs. We select the $\alpha$ that maximizes the improvement in refusal score, subject to the Over-Refusal Score remaining within $[-13, -12]$ and perplexity not exceeding $6$, because the original model without intervention achieves a perplexity of $5.002$ and an Over-Refusal of $-13.219$.

Table~\ref{tab:utility} reports the evaluation results across different ranks for Quadratic Probing on Llama-2-7b. Overall, the Quadratic method preserves general model utility and safety across all ranks, with MMLU scores remaining stable in the range of 47.20–47.34\%, indicating that the refusal-baking procedure introduces minimal degradation on general language understanding. For mathematical reasoning (GSM8K) and science question answering (ARC-Challenge), higher ranks tend to yield better performance. Meanwhile, rank=4 achieves the strongest safety performance on TAP (1.0) and GPTFuzz (0.5) with competitive utility scores, offering a favorable balance between safety and utility.

\begin{table}[t] 
\centering
\small 
\caption{Activation steering results across different positions and layers. The baseline corresponds to $\alpha=0$.} \label{tab:activation_steering} 
\begin{tabular}{llrrrr}
\toprule
Position & Layer & $\alpha$ & MC1 $\uparrow$ & MC2 $\uparrow$ & CE $\downarrow$ \\
\midrule
residual\_pre\_attn             & 23 &  8  & 0.3000 & 0.2606 & 1.6604 \\
input\_layernorm\_out           & 31 & -8  & 0.2667 & 0.2523 & 1.6586 \\
attn\_output                    & 15 & -2  & 0.2667 & 0.2555 & 1.6617 \\
residual\_mid                   & 24 & 16  & 0.2833 & 0.2546 & 1.6603 \\
post\_attention\_layernorm\_out & 31 &  4  & 0.2833 & 0.2583 & 1.6579 \\
mlp\_output                     & 31 & 24  & 0.2667 & 0.2588 & \textbf{1.6470} \\
residual\_out                   & 31 &  2  & 0.2833 & 0.2556 & 1.6614 \\
\bottomrule
\end{tabular}
\end{table}

\subsection{Qualitative Steering Examples on AxBench}
\label{appx:axbench_qual}

To complement the quantitative AxBench results, we qualitatively compare the generations produced by each steering method for the same concept, the same input prompt, and a matched steering strength. 

\begin{tcolorbox}[
    colback=gray!5,
    colframe=gray!70,
    title=\textbf{Box 3: Example 1 -- [references and citations within text]},
    fonttitle=\bfseries,
    breakable,
    arc=2mm
]
\small
\textbf{Model:} [google/gemma-2-2b-it] \quad
\textbf{Layer / Position:} [10] \quad
\textbf{Direction:} [increase concept]

\vspace{0.5em}
\textbf{Prompt:} \\
{[Create a web scraper that extracts specific data from a given website
]}

\vspace{0.5em}
\renewcommand{\arraystretch}{1.2}
\begin{tabular}{@{}>{\raggedright\arraybackslash}p{0.22\linewidth} >{\raggedright\arraybackslash}p{0.70\linewidth}@{}}
\toprule
\textbf{Method} & \textbf{Generated Response} \\
\midrule
No steering (baseline) & [I can't create a web scraper for you directly. I'm a text-based AI, so I can't interact with websites or execute code. 

However, I can provide you with a conceptual outline and code snippets to get you started. 

**Understanding Web Scraping**

Web scraping] \\
DiffMeans               & [Let's dive into the world of web scraping! 

** What is web scraping?**

 Web scraping is the process of collecting data from websites using a variety of techniques and tools. It's a broad term encompassing a wide range of techniques and methods, but it's essentially about extracting information] \\
Linear Probe             & [It's impossible to give you a single, universal "web scraper" because there are many different ways to scrape data from websites, and the best approach depends on what you're trying to achieve. 

Here's a breakdown of some common web scraping techniques and what they're used for:] \\
Quadratic Probe (ours)   & [\textbf{Let's talk about web scraping!} 

**What is web scraping?**

Web scraping is the practice of extracting data from websites using automated tools and techniques. It's a powerful tool for gathering information, but it's not just about "scraping" - it's about understanding the different ways] \\
\bottomrule
\end{tabular}
\end{tcolorbox}

\begin{tcolorbox}[
    colback=gray!5,
    colframe=gray!70,
    title=\textbf{Box 4: Example 2 -- [references and citations within text]},
    fonttitle=\bfseries,
    breakable,
    arc=2mm
]
\small
\textbf{Model:} [google/gemma-2-2b-it] \quad
\textbf{Layer / Position:} [10] \quad
\textbf{Direction:} [increase concept]

\vspace{0.5em}
\textbf{Prompt:} \\
{[Olivia earns \$9 per hour.  She worked 4 hours on Monday, 3 hours on Wednesday and 6 hours on Friday.  How much money did Olivia make this week?
]}

\vspace{0.5em}
\renewcommand{\arraystretch}{1.2}
\begin{tabular}{@{}>{\raggedright\arraybackslash}p{0.22\linewidth} >{\raggedright\arraybackslash}p{0.70\linewidth}@{}}
\toprule
\textbf{Method} & \textbf{Generated Response} \\
\midrule
No steering (baseline) & [Here's how to calculate Olivia's earnings:

* **Monday's earnings:** \$9/hour * 4 hours = \$36
* **Wednesday's earnings:** \$9/hour * 3 hours = \$27
* **Friday's earnings:** \$9/hour *] \\
DiffMeans               & [**Olivia's earnings this week depend on the context.  **

** To answer your question, we need to know what kind of work is she doing?**

 **To get a better understanding of how much money Olivia makes, we need to know what kind of work she does.**] \\
Linear Probe             & [Olivia earned **\$9 per hour** x **4 hours** = **\$36** on Monday. 

To get the full picture, you need to consider the rest of the week:

* **Olivia earned \$9 per hour x 4 hours** = **\$36** on Wednesday. ] \\
Quadratic Probe (ours)   & [It's impossible to say for sure how much money Olivia made this week without knowing her full name and what kind of work she does. 

**To answer your question, we need to know more about Olivia's work:**

* **What kind of work does Olivia do?**  
* **] \\
\bottomrule
\end{tabular}
\end{tcolorbox}



\end{document}